\documentclass{easychair}

\usepackage[utf8]{inputenc}
\usepackage[T1]{fontenc}
\usepackage{url}
\usepackage{booktabs}
\usepackage{amsfonts}
\usepackage{nicefrac}
\usepackage{microtype}
\usepackage{graphicx}
\usepackage{subcaption}
\usepackage{amsmath}
\usepackage{amssymb}
\usepackage{mathtools}
\usepackage{multirow}
\usepackage{algorithm}
\usepackage[noend]{algpseudocode}
\algnewcommand\algorithmicto{\textbf{to}}
\algnewcommand\To{\algorithmicto\ }
\usepackage{tikz}
\usepackage[capitalize,noabbrev]{cleveref}

\usetikzlibrary{arrows.meta,positioning,fit,calc}

\theoremstyle{plain}
\newtheorem{theorem}{Theorem}[section]
\newtheorem{proposition}[theorem]{Proposition}
\newtheorem{lemma}[theorem]{Lemma}

\theoremstyle{definition}
\newtheorem{definition}[theorem]{Definition}
\newtheorem{assumption}[theorem]{Assumption}

\theoremstyle{remark}
\newtheorem{remark}[theorem]{Remark}

\title{MiSS: A Logic-Driven Explanation of Minimal Sufficient Coalitions for Point Cloud Classifiers}

\author{Mengda Xing \and Jean-Marie Lagniez}
\institute{CRIL, Univ. Artois, Lens, France\\
\email{mengda.xing@univ-artois.fr, jmarie.lagniez@univ-artois.fr}}

\authorrunning{M. Xing and J.-M. Lagniez}
\titlerunning{MiSS}

\begin{document}

\maketitle

\begin{abstract}
	We present MiSS, a black-box, query-based framework for explaining 3D point cloud classifiers through perturbation-relative sufficiency reasoning. MiSS treats a superpoint partition as an interpretable abstraction layer and asks whether the original prediction can be certified from a minimal coalition of geometric regions under a specified perturbation distribution. Unlike abductive explainers that require Boolean feature spaces or white-box logical encodings of the predictor, MiSS separates candidate proposal from verification: a weighted MaxSAT procedure proposes coalitions using a heuristic adaptive cardinality floor, certified exact-size fallback, a safely tightened upper bound, blocking clauses, and a surrogate acquisition heuristic learned from previous oracle evaluations, while a black-box statistical oracle decides sufficiency from prediction queries. The system returns a statistically verified sufficient coalition as a binary attribution, with minimum cardinality guaranteed when certified search completes. Experiments on ModelNet40 and ShapeNet with PointNet and PointMLP classifiers show higher precision and coverage than rule-based baselines in most settings, with lower explanation time than exhaustive search.
\end{abstract}

\section{Introduction}

Explainable AI (XAI) is used to audit deep learning models in domains where model errors carry operational risk~\cite{gradcam,adebayo2018sanity,sundararajan2017axiomatic}.
One recurring difficulty is that explanations must be faithful to the model while remaining interpretable to humans.
This issue appears in perception tasks, where models operate on high-dimensional, low-level signals~--~such as pixel grids or geometric coordinates~--~while users often reason in terms of objects and parts.
Standard feature attribution methods~\cite{shap,lime} may identify scattered pixels or points that influence a prediction, but such attributions do not by themselves provide a sufficient reason for the model's output.

Recent work has therefore studied \textit{Abductive Explanations} (or Prime Implicants), which formalize sufficient reasons for model predictions.
While attribution methods like SHAP~\cite{shap} or LIME~\cite{lime} answer ``How much does this feature contribute?'', abductive methods ask whether a subset of features is sufficient for the prediction.
Formally, an abductive explanation identifies a minimal subset of features whose presence serves as a \textit{sufficient certificate}: once these features are fixed, the model's prediction is guaranteed to remain invariant, regardless of any perturbation to the remaining inputs~\cite{ignatiev2020contrastive}.
Despite the success of abductive reasoning in tabular and Boolean domains~\cite{marques2023logic,darwiche2023complete}, extending these guarantees to deep perception models remains an open challenge.

The difficulty comes from applying discrete logical reasoning to continuous neural models \cite{liang2025ai}.
Deep neural networks are non-linear predictors, and common black-box settings do not expose a logical representation that can be verified directly \cite{zhang2021survey}.
For such models, verifying sufficiency over a continuous input space is computationally difficult.
As a result, many practical explainers rely on perturbation or attribution heuristics without a minimal sufficiency certificate.

This limitation is acute in 3D computer vision.
Unlike 2D images with their regular grid, 3D point clouds are sparse, unordered, and irregular.
Point cloud classification underpins applications such as autonomous driving and robotics~\cite{liu2023point}, yet explaining these classifiers remains difficult: gradient-based explainers (e.g., Grad-CAM~\cite{gradcam}) rely on convolutional topologies that do not map cleanly to the MLPs or Transformers of modern 3D architectures~\cite{qi2017pointnet,zhao2021point}.
Signal-level explanations also face a large search space, as highlighting individual points in a cloud of thousands reveals little about which geometric regions preserve the decision.

A further drawback of point-level explanations is their size, which may exceed users' cognitive capacity, whereas XAI users generally prefer concise explanations~\cite{mil56,SAATY2003233,DBLP:conf/hcomp/LageCHNKGD19}. We therefore use an existing superpoint partition as an abstraction layer, reducing explanation search from thousands of points to a small set of geometric regions. Our contribution is not segmentation itself, but a black-box framework for finding minimal prediction-preserving superpoint coalitions.

MiSS bridges the gap between logically certified abductive explanations and uncertified point-cloud attribution maps. It defines perturbation-relative sufficiency and uses weighted MaxSAT to propose promising superpoint coalitions, which are verified by a statistical black-box oracle. Each verified solution tightens the cardinality upper bound, focusing subsequent search on shorter explanations. The main contributions of this work are as follows:
\begin{itemize}
	\item \textbf{Perturbation-Relative Sufficient Explanations:} We define a Minimal Sufficient Set over an interpretable superpoint partition using a probabilistic sufficiency criterion. This formulation makes the scope of the explanation explicit through the input, partition, perturbation distribution, and precision threshold, and applies to continuous black-box classifiers without requiring a logical model encoding.
	\item \textbf{Surrogate-Guided MaxSAT Search with Oracle Verification:} We separate candidate proposal from sufficiency verification. Weighted MaxSAT combines sound blocking constraints with soft clauses learned from the complete oracle history, while a sequential black-box oracle makes the final sufficiency decision.
	\item \textbf{Adaptive Cardinality Constraints:} We accelerate search using a safely tightened upper bound and a heuristic adaptive lower bound. Certified fallback recovers skipped sizes, while an UNSAT certificate guarantees global minimum cardinality once the exhausted window covers the level below the incumbent.
\end{itemize}

\section{Preliminaries~\label{sec:Preliminaries}}

\subsection{Formal Setting\label{sec:formal_setting}}

Let $\mathcal{X}$ be the space of point clouds, where an input $P \in \mathcal{X}$ is a set of $N$ points $\{p_1, \dots, p_N\}$ with $p_i \in \mathbb{R}^d$ (typically $d=3$ for XYZ coordinates).
We consider a black-box classifier $f: \mathcal{X} \to \mathcal{Y}$ that maps a point cloud $P$ to a predicted class label $y \in \mathcal{Y}$.
To reduce the explanation vocabulary to interpretable geometric units, we impose a partition on $P$.

\begin{definition}[Geometric Segmentation]
A \emph{segmentation} is a partition $\mathcal{S} = \{S_1, \dots, S_K\}$ of $P$ such that $\bigcup_{i} S_i = P$ and $S_i \cap S_j = \emptyset$ for $i \neq j$. Each segment $S_i$ is referred to as a \emph{superpoint}.
\end{definition}

We treat each superpoint as an atomic explanation unit that is either retained at its original geometry or replaced by the perturbation operator. A \textit{coalition} $T \subseteq \mathcal{S}$ denotes the retained subset; the complement $\mathcal{S} \setminus T$ is replaced.

\begin{definition}[Perturbation Operator]
A \emph{perturbation operator} $\mathcal{M}$ is a stochastic function that produces a perturbed sample $\tilde{P}_T$ by retaining the original geometry of the superpoints in $T$ and replacing each superpoint outside $T$ with a random sample:
\begin{equation}
    \tilde{P}_T = \left(\bigcup_{S_i \in T} S_i\right) \cup \left(\bigcup_{S_j \notin T} \mathcal{M}(S_j)\right)
\end{equation}
The operator must produce valid point clouds; in this work, $\mathcal{M}$ replaces each masked superpoint with a patch drawn from an in-distribution patch bank.
\end{definition}

The validity of the resulting sufficiency test is tied to this perturbation distribution. In point clouds, the operator should preserve plausible local geometry; otherwise, the oracle may evaluate off-distribution artifacts rather than the absence of masked superpoints. The granularity of $\mathcal{S}$ also sets the resolution of the explanation.

This perturbation model supports the following classical sufficiency condition. The universal quantification it uses is intractable for continuous black-box models; Section~\ref{sec:method} replaces it with a probabilistic criterion that can be evaluated by a finite sequence of model queries.

\begin{definition}[Abductive Explanation]
Let $f$ be a classifier mapping an input to a class $y$. An \emph{abductive explanation} is a minimal coalition $\mathcal{E} \subseteq \mathcal{S}$ that satisfies:
\begin{enumerate}
    \item \textbf{Sufficiency:} $f(\tilde{P}_{\mathcal{E}}) = y$ for any realization of the perturbation on $\mathcal{S} \setminus \mathcal{E}$.
    \item \textbf{Minimality:} No proper subset of $\mathcal{E}$ satisfies the sufficiency condition.
\end{enumerate}
\end{definition}

\subsection{Propositional Logic and MaxSAT-based Search~\label{sec:sat_preliminaries}}

We use standard propositional notation. A literal is a variable $x$ or its negation $\bar{x}$, a clause is a disjunction of literals, and a CNF formula is a conjunction of clauses. An interpretation $\omega$ is a model of a formula $\Sigma$, written $\omega \models \Sigma$, when it satisfies all clauses of $\Sigma$.

SAT asks whether a CNF formula admits a model, and cardinality constraints such as $\sum_{i=1}^{n} x_i \le L$ can be encoded into CNF~\cite{biere2009handbook,DBLP:series/faia/RousselM21}. Weighted MaxSAT separates clauses into hard constraints and weighted soft clauses; the solver must satisfy every hard clause and maximize the total satisfied soft weight. In our setting, hard clauses encode adaptive at-most cardinality constraints and blocking clauses, while soft clauses encode acquisition preferences estimated from previous oracle evaluations.

\section{The MiSS Framework\label{sec:method}}

Figure~\ref{fig:framework} summarizes the MiSS workflow. The input point cloud is first partitioned into superpoints, and the singleton coalitions are evaluated by the oracle to seed the search history and initialize the per-superpoint surrogate weights. The surrogate-guided MaxSAT solver then proposes candidate coalitions, each of which is evaluated by the black-box sufficiency oracle under the perturbation distribution. The result is fed back to update the search history, blocking clauses, cardinality bounds, and per-superpoint surrogate weights. If multiple sufficient coalitions are verified within the search budget, the minimum-cardinality one is selected and returned as a binary superpoint attribution.

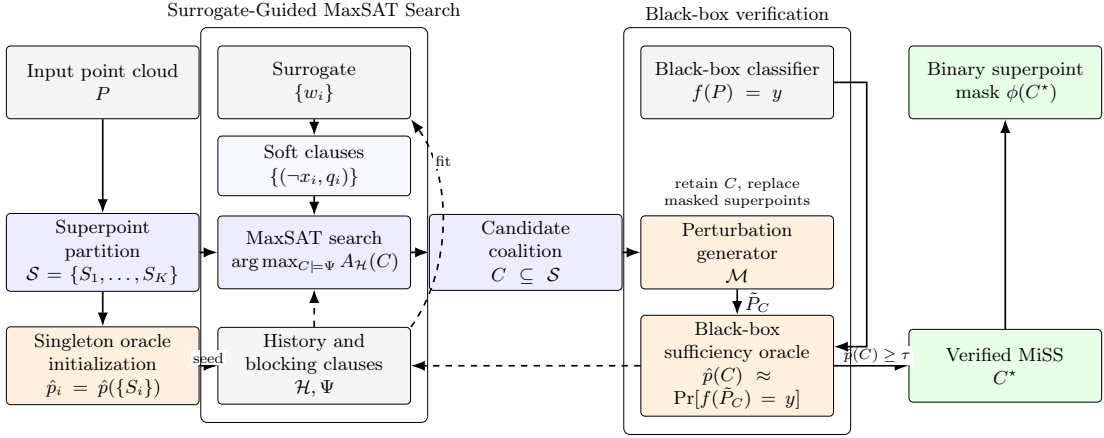
\begin{figure*}[!t]
	\centering
	\resizebox{\textwidth}{!}{%
		\begin{tikzpicture}[
				font=\small,
				>=Latex,
				block/.style={
						draw,
						rounded corners=2pt,
						align=center,
						text width=3.0cm,
						minimum height=1.22cm,
						fill=gray!8
					},
				process/.style={
						draw,
						rounded corners=2pt,
						align=center,
						text width=3.0cm,
						minimum height=1.22cm,
						fill=blue!7
					},
				oracle/.style={
						draw,
						rounded corners=2pt,
						align=center,
						text width=3.0cm,
						minimum height=1.22cm,
						fill=orange!12
					},
				output/.style={
						draw,
						rounded corners=2pt,
						align=center,
						text width=3.0cm,
						minimum height=1.22cm,
						fill=green!10
					},
				note/.style={
						align=center,
						font=\scriptsize,
						text width=2.95cm
					},
				softclause/.style={
						draw,
						rounded corners=2pt,
						align=center,
						text width=3.0cm,
						minimum height=1.0cm,
						fill=blue!3
					},
				arrow/.style={->, thick},
				feedback/.style={->, thick, dashed}
			]
			\node[block] (input) at (0,2.85) {Input point cloud\\$P$};
			\node[process] (seg) at (0,0) {Superpoint\\partition\\$\mathcal{S}=\{S_1,\ldots,S_K\}$};
			\node[oracle] (singleton) at (0,-1.9) {Singleton oracle\\initialization\\$\hat{p}_i=\hat{p}(\{S_i\})$};
			\node[process] (search) at (3.55,0) {MaxSAT search \\ $\arg\max_{C \models \Psi} A_\mathcal{H}(C)$};
			\node[process] (candidate) at (7.1,0) {Candidate\\coalition\\$C\subseteq\mathcal{S}$};
			\node[oracle] (perturb) at (10.65,0) {Perturbation\\generator\\$\mathcal{M}$};

			\node[block] (surrogate) at (3.55,2.85) {Surrogate\\$\{w_i\}$};
			\node[softclause] (softclauses) at (3.55,1.45) {Soft clauses\\$\{(\neg x_i,q_i)\}$};
			\node[block] (history) at (3.55,-1.9) {History and blocking clauses\\$\mathcal{H},\Psi$};
			\node[block] (classifier) at (10.65,2.85) {Black-box classifier\\$f(P)=y$};

			\node[oracle] (oracle) at (10.65,-1.9) {Black-box\\sufficiency oracle\\$\hat{p}(C)\approx\Pr[f(\tilde{P}_C)=y]$};
			\node[output] (mask) at (15.15,2.85) {Binary superpoint\\mask $\phi(C^\star)$};
			\node[output] (output) at (15.15,-1.9) {Verified MiSS\\$C^\star$};

			\draw[arrow] (input.south) -- (seg.north);
			\draw[arrow] (seg.east) -- (search.west);
			\draw[arrow] (seg.south) -- (singleton.north);
			\draw[feedback] (singleton.east) -- node[above, font=\scriptsize, fill=white, inner sep=1pt] {seed} (history.west);
			\draw[arrow] (search.east) -- (candidate.west);
			\draw[arrow] (candidate.east) -- (perturb.west);
			\draw[arrow] (perturb.south) -- node[right, font=\small] {$\tilde{P}_C$} (oracle.north);
			\draw[arrow] (oracle.east) -- node[pos=0.58, above, font=\scriptsize, fill=white, inner sep=1pt] {$\hat{p}(C)\geq\tau$} (output.west);
			\draw[arrow] (output.north) -- (mask.south);
			\draw[arrow] (classifier.east) -- ++(0.55,0) |- ($(oracle.east)+(0,0.32)$);

			\draw[feedback] (oracle.west) -- (history.east);
			\draw[feedback] (history.north) -- (search.south);
			\draw[feedback] (history.north east) .. controls +(0.7,0.95) and +(0.7,-0.95) .. node[pos=0.78, right, font=\scriptsize, fill=white, inner sep=1pt] {fit} (surrogate.south east);
			\draw[arrow] (surrogate.south) -- (softclauses.north);
			\draw[arrow] (softclauses.south) -- (search.north);

			\node[note] (pert-note) at (10.65,1.0) {retain $C$, replace\\masked superpoints};
			\node[
				draw,
				rounded corners=3pt,
				fit=(search) (surrogate) (softclauses) (history),
				inner xsep=0.28cm,
				inner ysep=0.3cm,
				label={[font=\small]above:Surrogate-Guided MaxSAT Search}
			] {};
			\node[
				draw,
				rounded corners=3pt,
				fit=(perturb) (oracle) (classifier),
				inner xsep=0.28cm,
				inner ysep=0.3cm,
				label={[font=\small]above:Black-box verification}
			] {};
		\end{tikzpicture}%
	}
	\caption{Overview of the MiSS framework. The input point cloud is partitioned into superpoints, and singleton coalitions are first evaluated by the oracle to seed the search history and initialize the per-superpoint surrogate weights. Candidate coalitions are then proposed by a MaxSAT search guided by the surrogate and verified by the black-box sufficiency oracle under the perturbation distribution. Oracle outcomes update the search history and blocking clauses; whenever a sufficient coalition is verified, the current cardinality upper bound is tightened, yielding a MiSS once all proper subcoalitions of the returned candidate have been verified as insufficient or safely pruned.}
	\label{fig:framework}
\end{figure*}

\subsection{Minimal Sufficient Sets\label{sec:miss_def}}

Classical abductive explanations (Prime Implicants / AXp) define a sufficient reason as a minimal set of features whose fixed values \emph{force} the prediction under any assignment to the remaining features~\cite{ignatiev2020contrastive}. This $\forall$-quantified condition cannot be checked by prediction queries alone over a continuous input space: the complement is infinite and a black-box model exposes no logical encoding. We therefore replace the universal condition with a probabilistic criterion parameterised by a user-chosen threshold $\tau$ and a perturbation distribution, making sufficiency verifiable by a finite sequence of model queries.

\begin{definition}[Minimal Sufficient Set (MiSS)\label{def:miss}]
	A coalition $C \subseteq \mathcal{S}$ is a \emph{Minimal Sufficient Set} for target class $y$ under threshold $\tau \in (0,1]$ if:
	\begin{enumerate}
		\item \textbf{Sufficiency:} $\Pr[f(\tilde{P}_{C}) = y] \geq \tau$.
		\item \textbf{Minimality:} $\forall C' \subsetneq C,\ \Pr[f(\tilde{P}_{C'}) = y] < \tau$.
	\end{enumerate}
\end{definition}

\begin{assumption}[Monotonic Sufficiency Assumption\label{ass:monotonic}]
	We assume that the sufficiency predicate is monotonic with respect to set inclusion:
	$S\subseteq T,\ \mathrm{sufficient}(S) \Rightarrow \mathrm{sufficient}(T)$.	
	Under this assumption, checking all one-feature removals of a coalition is enough to certify subset minimality, avoiding enumeration of all proper subsets. When this assumption is violated, this check certifies only one-step minimality.
\end{assumption}

A MiSS is a local, post-hoc explanation: its guarantee is relative to the fixed input $P$, partition $\mathcal{S}$, perturbation operator, and threshold $\tau$, and does not claim causal necessity outside that scope.

\begin{proposition}[Existence\label{prop:existence}]
	If any coalition $T \subseteq \mathcal{S}$ satisfies $\Pr[f(\tilde{P}_T) = y] \geq \tau$, then at least one MiSS exists.
\end{proposition}
\begin{proof}
	The family of all $\tau$-sufficient subsets of $T$ is finite (it lies in $2^{T}$) and non-empty (it contains $T$ itself). Let $C$ be any element of minimum cardinality. Then $C$ is $\tau$-sufficient by construction. Moreover, any $C' \subsetneq C$ is also a subset of $T$; were it $\tau$-sufficient it would be an element of the family of strictly smaller cardinality, contradicting the minimality of $C$. Hence $C$ satisfies both conditions of Definition~\ref{def:miss}.
\end{proof}

\begin{algorithm}[!t]
	\caption{MiSS Explainer}
	\label{alg:localmiss}
	\begin{algorithmic}[1]

		\Require Point cloud $P$, black-box classifier $f$, target class $c=f(P)$, superpoint set
		$\mathcal{S}$, perturbation generator $\mathcal{M}$, max coalition size $L$, heuristic floor patience $\eta$, precision threshold $\tau$, confidence level $\delta_{\mathrm{test}}$, batch size $B$, max test samples $N_{\max}$, search budget $Q$, max time $T_{\max}$
		\Ensure Selected verified coalition $C^\star$ or $\emptyset$

		\State Initialize $\Psi\leftarrow\top$, $\mathcal{H},\mathcal{R}\leftarrow\emptyset$, $F[\ell]\leftarrow0$, and $[\ell_{\mathrm{cur}},L_{\mathrm{cur}}]\leftarrow[1,L]$
		\ForAll{$S_i\in\mathcal{S}$} \Comment{surrogate initialization}
		\State $(b_i,\hat{p}_i,n_i)\leftarrow\textsc{IsSufficient}(\{S_i\},c,\mathcal{M})$
		\State Append $(\{S_i\},\hat{p}_i,n_i)$ to $\mathcal{H}$ and record it in $\mathcal{R}$ if $b_i$
		\EndFor
		\State $\ell_{\mathrm{cur}}\leftarrow\textsc{InitialFloor}(\mathcal H,L_{\mathrm{cur}},\tau)$
		\State Fit the initial surrogate on $\mathcal{H}$ and compute per-superpoint weights $w_i$~\eqref{eq:ridge}--\eqref{eq:acquisition}
		\While{time $< T_{\max}$ and $|\mathcal{H}| < Q$}
		\State Build WCNF $\Phi_t$ from $\Psi_t=\Psi\land(\ell_{\mathrm{cur}}\leq\sum_i x_i\leq L_{\mathrm{cur}})$~\eqref{eq:cardinality_window} and soft clauses by $A_\mathcal{H}$
		\State Get candidate coalition $C \subseteq \mathcal{S}$ from MaxSAT solver on $\Phi_t$
		\If{$C=\bot$}
		\textbf{break} \Comment{UNSAT, no more candidates}
		\EndIf
		\State $(b,\,\hat{p},\,n) \leftarrow$ \textsc{IsSufficient}$(C, c, \mathcal{M})$ \Comment{Algorithm~\ref{alg:is_sufficient}}
		\State Append $(C,\hat{p}, n)$ to history $\mathcal{H}$
		\If{$b$}
		\State Append $C$ to $\mathcal{R}$
		\State Append exact candidate-blocking clause to prevent proposing $C$ again
		\State $L_{\mathrm{cur}}\leftarrow |C| - 1$
		\Else
		\State Append failed-subset pruning clause $\Psi \leftarrow \Psi \land \left(\bigvee_{S_i \notin C} x_i\right)$
		\EndIf
		\State $\ell_{\mathrm{cur}}\leftarrow\textsc{UpdateFloor}(\ell_{\mathrm{cur}},L_{\mathrm{cur}},\eta,C,b)$  \Comment{Algorithm~\ref{alg:lower_bound_update}}
		\State Re-fit surrogate and update soft clauses \Comment{Algorithm~\ref{alg:surrogate_update}}
		\EndWhile
		\State \Return selected coalition $C^\star$ from $\mathcal{R}$, or $\emptyset$ if $\mathcal{R}=\emptyset$

	\end{algorithmic}
\end{algorithm}

\subsection{Surrogate-Guided MaxSAT Search\label{sec:algorithm}}

Finding a MiSS is in principle an exponential search over $2^K$ coalitions. We target:
\begin{equation}\label{eq:objective}
	C^\star \in \arg\min_{C \subseteq \mathcal{S}} |C|
	\quad \text{s.t.} \quad
	\Pr[f(\tilde{P}_{C}) = y] \geq \tau.
\end{equation}
Algorithm~\ref{alg:localmiss} alternates surrogate-guided MaxSAT proposal with black-box sufficiency verification. It first evaluates all singleton coalitions to initialize the oracle history and per-superpoint surrogate weights. At each iteration, MaxSAT combines cardinality and blocking constraints with surrogate-derived soft clauses to propose a candidate; with uniform positive soft weights, this reduces to smallest-first enumeration. The oracle result then updates the history, constraints, cardinality bounds, and surrogate. The search continues until exhaustion or budget termination and returns the minimum-cardinality coalition among the verified sufficient candidates.

\begin{algorithm}[!t]
	\caption{Surrogate Fitting and Soft-Clause Construction}
	\label{alg:surrogate_update}
	\begin{algorithmic}[1]

		\Require History $\mathcal{H}=\{(C_j,\hat{p}_j,n_j)\}_{j=1}^{m}$, features $\mathcal{S}$, penalty $\lambda_{\mathrm{ridge}}$, scale $\alpha_{\mathrm{soft}}$
		\Ensure Surrogate-induced soft clauses for the next WCNF

		\State Build a binary design matrix $X$ from $\mathcal{H}$ using an intercept and the indicator vector of each coalition $C_j$
		\State Set targets $y_j \leftarrow \hat{p}_j$
		\State Set regression sample weights $r_j \leftarrow n_j$
		\State Fit a sample-weighted ridge surrogate on $(X,y)$
		\ForAll{features $S_i \in \mathcal{S}$}
		\State Let $e_i$ be the one-hot indicator vector encoding $\{S_i\}$
		\State Set the surrogate weight of $S_i$ to $w_i \leftarrow \mathrm{clip}(\hat{\theta}^{\top}[1\,;\,e_i],0,1)$
		\State Add soft clause $(\neg x_i,\,\alpha_{\mathrm{soft}}(1-w_i))$
		\EndFor
		\State \Return the updated WCNF soft clauses

	\end{algorithmic}
\end{algorithm}

\subsubsection*{Surrogate-Guided Proposal}
Algorithm~\ref{alg:surrogate_update} summarizes how the oracle history is converted into per-superpoint surrogate weights and the corresponding MaxSAT soft clauses. Each superpoint $S_i$ is associated with the weighted soft clause
\begin{equation}\label{eq:soft_clause}
	(\neg x_i,\,q_i),
	\qquad
	q_i=\operatorname{round}\!\left(\alpha_{\mathrm{soft}}(1-w_i)\right),
\end{equation}
where $x_i=1$ denotes that $S_i$ is included, $w_i\in[0,1]$ is the surrogate weight assigned to superpoint $S_i$, and $\alpha_{\mathrm{soft}}$ converts the penalty to an integer MaxSAT weight. The negative literal $\neg x_i$ is satisfied when the superpoint is omitted; selecting $S_i$ violates the clause and incurs penalty $q_i$. Thus every selected superpoint incurs a cost, encouraging shorter coalitions, while a superpoint with a larger weight $w_i$ incurs a smaller inclusion penalty. In the uniform case $q_i=q>0$ for all $i$, the total violated soft weight is exactly $q|C|$, so MaxSAT returns a smallest feasible coalition before considering larger ones.

However, exploring coalitions strictly level by level becomes expensive when the smallest MiSS is large. To reach larger coalitions sooner while still proposing meaningful candidates, we bias the search with the accumulated oracle history: the per-superpoint weights are learned from \emph{all} previous oracle evaluations rather than from the latest decision alone. Let $\mathcal{H}=\{{(C_j,\hat{p}_j,n_j)\}}_{j=1}^{m}$ contain every evaluated coalition, its empirical precision $\hat{p}_j$, and the number $n_j$ of oracle samples used to estimate it. The initialization phase contributes the precision of every singleton coalition (Algorithm~\ref{alg:localmiss}, lines 2--4), and each later MaxSAT proposal adds another observation (Algorithm~\ref{alg:localmiss}, line 12). We fit the sample-weighted ridge surrogate
\begin{equation}\label{eq:ridge}
	\hat{\theta}
	=
	\arg\min_{\theta}
	\sum_{j=1}^{m} n_j\bigl{(\hat{p}_j - \theta^\top [1\,;\,z(C_j)]\bigr)}^2
	+
	\lambda_{\mathrm{ridge}} \lVert \theta_{1:K} \rVert_2^2,
\end{equation}
where ${z(C_j)\in\{0,1\}}^K$ is the binary indicator of $C_j$ and the intercept is not regularized. Reusing the complete history avoids wasting expensive oracle queries. Crucially, the regression target is the normalized precision $\hat{p}_j$ rather than an absolute failure count, so a promising coalition is not penalized simply for having been tested more often; the sample count $n_j$ instead controls how strongly its precision estimate is trusted.

The surrogate-weight vector $\mathbf{w}=(w_1,\ldots,w_K)$ used in Equation~\eqref{eq:soft_clause} is computed by
\begin{equation}\label{eq:acquisition}
	\mathbf{w}=\mathrm{clip}\!\left(\hat{\theta},0,1\right).
\end{equation}

Together with $\Psi_t$, the MaxSAT solver returns
\begin{equation}\label{eq:proposal}
	C^{\mathrm{prop}}_t \in \arg\max_{C \models \Psi_t} A_\mathcal{H}(C),
\end{equation}
where $A_\mathcal{H}$ is the total satisfied soft weight induced by the oracle-history surrogate.

\subsubsection*{Heuristic Adaptive Cardinality Floor}

In the same spirit as the surrogate above, which steers the search toward larger coalitions, we add a heuristic that avoids examining coalition sizes strictly one level at a time. Although users prefer concise explanations, overly small coalitions are often insufficient; MiSS therefore uses an adaptive cardinality window to skip unpromising sizes and accelerate the search. A further benefit is that whole groups of small coalitions can be discarded at once: if a coalition of size $k$ is found insufficient, then under the Monotonic Sufficiency Assumption (Assumption~\ref{ass:monotonic}) none of its subsets can be sufficient either, so all of them are pruned simultaneously.

\begin{algorithm}[!t]
	\caption{Cardinality-Floor Updates}
	\label{alg:lower_bound_update}
	\begin{algorithmic}[1]
		\Function{InitialFloor}{$\mathcal{H},L_{\mathrm{cur}},\tau$}
		\State $p_{\max}\leftarrow\max\{\hat p_j:(C_j,\hat p_j,n_j)\in\mathcal H,\ |C_j|=1\}$
		\State \Return $\min\!\left(L_{\mathrm{cur}},\max\!\left(2,\left\lceil\tau/\max(p_{\max},\varepsilon)\right\rceil\right)\right)$
		\EndFunction
		\Statex
		\Statex \textbf{State variable:} $F[\ell]$ counts oracle-rejected candidates of size $\ell$
		\Function{UpdateFloor}{$\ell_{\mathrm{cur}},L_{\mathrm{cur}},\eta,C,b$}
		\If{$\neg b$ and $|C|=\ell_{\mathrm{cur}}$}
		\State $F[\ell_{\mathrm{cur}}]\leftarrow F[\ell_{\mathrm{cur}}]+1$
		\If{$F[\ell_{\mathrm{cur}}]\geq \eta$}
		$\ell_{\mathrm{cur}}\leftarrow \ell_{\mathrm{cur}}+1$
		\EndIf
		\EndIf
		\If{$\ell_{\mathrm{cur}} \ge L_{\mathrm{cur}}$}
		$\ell_{\mathrm{cur}} = 2$
		\EndIf
		\State \Return $\ell_{\mathrm{cur}}$
		\EndFunction
	\end{algorithmic}
\end{algorithm}

Algorithm~\ref{alg:lower_bound_update} summarizes the heuristic floor updates. During the main search, MaxSAT restricts candidate size to the current window (Algorithm~\ref{alg:localmiss}, line 8)
\begin{equation}\label{eq:cardinality_window}
	\Psi_t = \Psi \wedge
	\left(\ell_{\mathrm{cur}} \leq \sum_{i=1}^{K} x_i \leq L_{\mathrm{cur}}\right),
	\qquad 1 \leq \ell_{\mathrm{cur}} \leq L_{\mathrm{cur}} \leq L.
\end{equation}
The lower bound $\ell_{\mathrm{cur}}$ is heuristic: it is initialized from the singleton history and may be increased after $\eta$ rejections, while the upper bound $L_{\mathrm{cur}}$ is tightened whenever a sufficient coalition is found (Algorithm~\ref{alg:localmiss}, line 16). Because the floor may skip a shorter sufficient coalition, the condition $\ell_{\mathrm{cur}} \ge L_{\mathrm{cur}}$ does not immediately terminate the algorithm; instead, the search returns to the minimum lower bound so that the full search space can still be explored. This revisiting is what allows the search to obtain an UNSAT window covering the level below the incumbent, which by Lemma~\ref{lem:unsat_window} certifies global minimum cardinality.

\subsubsection*{Oracle Verification and Clause Update}

\begin{algorithm}[!t]
	\caption{IsSufficient}
	\label{alg:is_sufficient}
	\begin{algorithmic}[1]

		\Require Coalition $C$, target class $c$, perturbation generator $\mathcal{M}$, precision threshold $\tau$, confidence level $\delta_{\mathrm{test}}$, batch size $B$, max test samples $N_{\max}$
		\Ensure Decision bit $b$, empirical precision estimate $\hat{p}(C)$, and number $n$ of oracle samples used

		\State $n_{\mathrm{succ}} \leftarrow 0$, $n \leftarrow 0$
		\If{$C = \emptyset$}
			\State \Return $(\texttt{False}, 0, 0)$
		\EndIf

		\While{$n < N_{\max}$}
			\State Sample $\{(m^{(b)}, \tilde{P}_C^{(b)})\}_{b=1}^{r}$ from $\mathcal{M}$ conditioned on coalition $C$, where $1 \le r \le B$
			\If{no perturbed samples are produced}
				\textbf{break}
			\EndIf
			\State Evaluate $\hat{y}^{(b)} \leftarrow f(\tilde{P}_C^{(b)})$ for all $b=1,\dots,r$
			\State $n_{\mathrm{succ}} \leftarrow n_{\mathrm{succ}} + \sum_{b=1}^{r} \mathbf{1}[\hat{y}^{(b)} = c]$
			\State $n \leftarrow n + r$
			\State $(\textit{prec\_lb}, \textit{prec\_ub}) \leftarrow \mathrm{KLBounds}(n_{\mathrm{succ}}, n, \delta_{\mathrm{test}})$
			\If{$\textit{prec\_lb} > \tau$}{ \Return $(\texttt{True}, n_{\mathrm{succ}} / n, n)$}
			\EndIf
			\If{$\textit{prec\_ub} < \tau$}
				\Return $(\texttt{False}, n_{\mathrm{succ}} / n, n)$
			\EndIf
		\EndWhile

		\If{$n = 0$}
			\Return $(\texttt{False}, 0, 0)$
		\EndIf
		\State \Return $((n_{\mathrm{succ}} / n) \ge \tau,\; n_{\mathrm{succ}} / n,\; n)$

	\end{algorithmic}
\end{algorithm}

Algorithm~\ref{alg:is_sufficient} applies a sequential KL-bounds test to the candidate $C$: it draws batches of perturbed samples $\tilde{P}_C$, queries $f$, and returns \texttt{True} as soon as the KL lower bound on the empirical precision exceeds $\tau$, or \texttt{False} as soon as the upper bound falls below $\tau$. If neither bound resolves within $N_{\max}$ samples the oracle falls back to the empirical ratio. Its output is the triple $(b,\hat{p},n)$, containing the sufficiency decision, empirical precision, and number of oracle samples used to estimate it. The oracle thus provides a \emph{statistically verified} sufficiency decision at confidence level $\delta_{\mathrm{test}}$.
When the oracle rejects $C$, a failed-subset pruning clause is appended:
	$\Psi \leftarrow \Psi \wedge \left(\bigvee_{S_i \notin C} x_i\right)$.

Under Assumption~\ref{ass:monotonic}, this clause safely removes $C$ and its subsets from future proposals: any future candidate must include at least one superpoint outside the rejected coalition. The triple $(C,\hat{p},n)$ is added to $\mathcal{H}$, an exact blocking clause is also added to avoid proposing the same candidate again, and the acquisition heuristic is re-fitted before the next MaxSAT call. If the oracle accepts $C$, the successful coalition is recorded and blocked by
$
	\Psi \leftarrow \Psi \wedge
	\left[
		\left(\bigvee_{S_i\in C}\neg x_i\right)
		\vee
		\left(\bigvee_{S_i\notin C}x_i\right)
		\right]
$, and the upper cardinality bound is updated (Algorithm~\ref{alg:localmiss}, line 16).

\subsection{Correctness Guarantees\label{sec:correctness}}

Intuitively, rejected coalitions prune themselves and, by monotonicity, all their subsets, while accepted ones tighten the size bound. An UNSAT result over the current window therefore certifies that every smaller candidate is provably insufficient or already recorded, as the next lemma and theorem formalize.

\begin{lemma}[UNSAT Certificate\label{lem:unsat_window}]
	Consider the hard formula $\Psi_t=\Psi\land(\ell_{\mathrm{cur}}\leq\sum_i x_i\leq L_{\mathrm{cur}})$
	built at Algorithm~\ref{alg:localmiss}, line 8, with $\ell_{\mathrm{cur}}\le L_{\mathrm{cur}}$.
	If $\Psi_t$ is UNSAT, then no $\tau$-sufficient coalition of size at most $L_{\mathrm{cur}}$ exists.
\end{lemma}
\begin{proof}
	Throughout Algorithm~\ref{alg:localmiss} every accepted coalition has size strictly greater
	than $L_{\mathrm{cur}}$: the bound is set to $|C|-1$ whenever a coalition $C$ is accepted
	(line 16), and later acceptances only decrease it. Since $\Psi_t$ is UNSAT and every
	coalition of size in $[\ell_{\mathrm{cur}},L_{\mathrm{cur}}]$ satisfies the cardinality
	constraint, each such coalition falsifies a clause of $\Psi$. The clauses of $\Psi$ are
	failed-subset clauses $\bigvee_{S_i\notin R}x_i$, added when a coalition $R$ is rejected and
	falsified exactly by the subsets of $R$, and exact blocks, each falsified by a single
	accepted coalition.

	Suppose a sufficient coalition $C$ has $\ell_{\mathrm{cur}}\le|C|\le L_{\mathrm{cur}}$. As $C$
	is blocked, either $C\subseteq R$ for some rejected $R$, whence $R$ is sufficient by
	monotonicity (Assumption~\ref{ass:monotonic}), contradicting its rejection; or $C$ equals an
	accepted coalition, which is impossible since accepted coalitions have size
	$>L_{\mathrm{cur}}\ge|C|$. Hence no sufficient coalition lies in the band
	$[\ell_{\mathrm{cur}},L_{\mathrm{cur}}]$.

	Finally, suppose a sufficient coalition $D$ has $|D|<\ell_{\mathrm{cur}}$. Extend it to
	$T\supseteq D$ with $|T|=\ell_{\mathrm{cur}}$; by monotonicity $T$ is sufficient and lies in
	the band, contradicting the previous paragraph. Therefore no $\tau$-sufficient coalition of
	size at most $L_{\mathrm{cur}}$ exists.
\end{proof}

\begin{theorem}[Minimum-Cardinality Guarantee\label{thm:minimality}]
	Suppose Algorithm~\ref{alg:localmiss} returns a verified coalition $C^\star$
	after an UNSAT window satisfying the conditions of
	Lemma~\ref{lem:unsat_window}. Then $C^\star$ has minimum cardinality among
	all $\tau$-sufficient coalitions and is therefore a {MiSS}.
\end{theorem}
\begin{proof}
	The oracle verifies that $C^\star$ is $\tau$-sufficient. At termination
	$L_{\mathrm{cur}}=|C^\star|-1$ (Algorithm~\ref{alg:localmiss}, line 16), so by
	Lemma~\ref{lem:unsat_window} no $\tau$-sufficient coalition of size $<|C^\star|$ exists.
	Hence $C^\star$ has minimum cardinality, solves Equation~\eqref{eq:objective}, and satisfies
	Definition~\ref{def:miss}.
\end{proof}

\begin{remark}
	The heuristic floor may temporarily skip smaller candidates.  An UNSAT
	window nevertheless certifies optimality as soon as it reaches cardinality
	$|C^\star|-1$, because monotonicity lifts every hypothetical smaller sufficient
	coalition into that window.  If the search is interrupted by the time limit
	$T_{\max}$ or query budget $Q$ before this condition is met, the returned
	coalition remains statistically $\tau$-verified, but minimum cardinality is
	not guaranteed.
\end{remark}

\subsection{Returned Explanation\label{sec:output}}

Algorithm~\ref{alg:localmiss} records all sufficient coalitions verified before certified exhaustion or budget termination and selects a primary explanation $C^\star$ from these successful records. Under certified termination, Theorem~\ref{thm:minimality} guarantees minimum cardinality; under budget termination, $C^\star$ remains a statistically verified sufficient coalition. The returned coalition is presented as a binary superpoint attribution:
\begin{equation}
	\phi_i(C^\star)=
	\begin{cases}
		1, & S_i \in C^\star,    \\
		0, & S_i \notin C^\star.
	\end{cases}
\end{equation}
This binary vector identifies which geometric regions are sufficient to preserve the prediction under the specified perturbation distribution, and directly supports visualisation by highlighting the corresponding superpoints in the point cloud.

\section{Related Works}
\label{sec:related}

\textbf{Signal-Level Explanations.} Explanations for 3D point clouds adapt many techniques originally developed for images. Several studies have extended saliency analysis to point clouds, including gradient-based importance measures~\cite{pcgradcam,pointcloudsaliency}, perturbation-based deletion and insertion metrics~\cite{pointmasking}, and contrastive explanation methods~\cite{pointcontrastive}. These approaches typically produce signal-level attributions that can be noisy, unstable, and semantically fragmented. Although such explanations are often visually intuitive, they usually do not provide certified sufficiency guarantees under perturbation. By contrast, we treat superpoints as structured explanatory units and verify the prediction-preservation probability of the returned coalition through the oracle.

\textbf{Rule-Based and Concept-Level Explanations.} Rule-based explainers represent a prediction as a logical predicate over feature literals rather than as a continuous attribution score, and constitute the main baselines in our tabular study (Table~\ref{tab:exp_tabular}). Anchors~\cite{anchor} greedily searches for if-then predicates whose local precision exceeds a user-specified threshold on a perturbed neighborhood. LORE~\cite{guidotti2018local} fits a local decision tree on a genetically generated neighborhood and reads the root-to-leaf path as the factual explanation, while sibling leaves supply counterfactual rules. SkopeRules~\cite{skoperules} mines high-precision rules from a bagged ensemble of trees and deduplicates them by logical equivalence; CORELS~\cite{angelino2017learning} instead produces certifiably optimal rule lists through branch-and-bound search over a discretized feature space. These methods move toward discrete and semantically meaningful explanations, but they are typically designed for tabular feature literals rather than structured 3D regions. In contrast, MiSS uses surrogate-guided MaxSAT with a heuristic cardinality floor and certified fallback over superpoint coalitions, and verifies the returned coalition under the perturbation distribution.


\section{Experiment} \label{sec:experiment}
\begin{figure}[!t]
	\centering
	\includegraphics[width=1\linewidth]{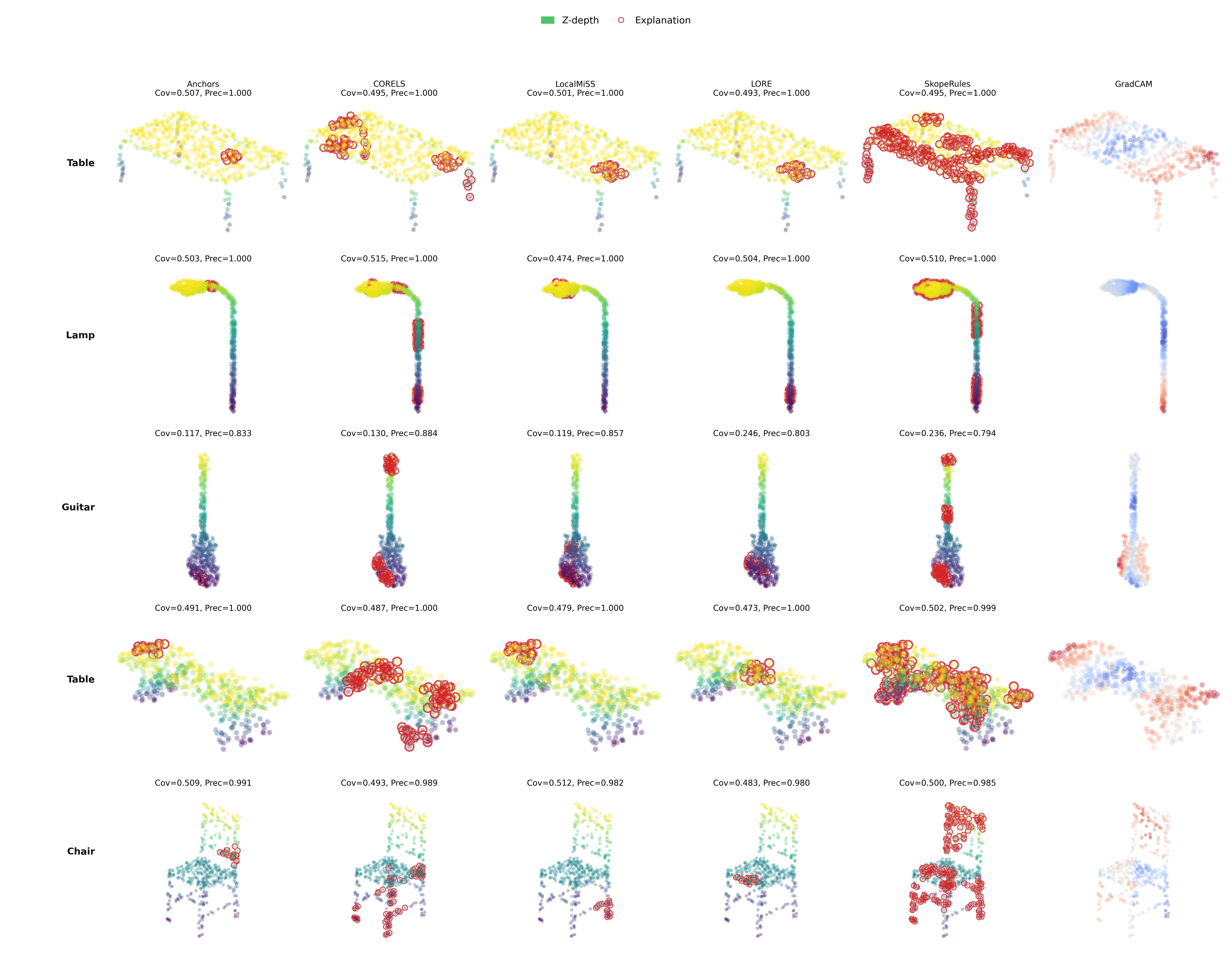}

		\caption{Comparison of superpoint-level explanations produced by different explainers. Columns correspond to Anchors, CORELS, MiSS, LORE, SkopeRules, and GradCAM. For all methods except GradCAM, the viridis colormap encodes point height along the $z$ axis and hollow red circles mark the explanation regions. GradCAM instead uses a signed coolwarm attribution map, where red indicates positive contribution and blue indicates negative contribution.}
	\label{fig:exp1}
\end{figure}

\subsection{Implementation Details}

\textbf{Datasets.} We evaluate on four datasets spanning 3D point cloud domains and tabular datasets. \emph{ModelNet40} is a large-scale 3D object classification benchmark containing 12,311 CAD models from 40 categories; we use the official split with 9,843 shapes for training and 2,468 for explaining the black-box model. \emph{ShapeNet} is a richly annotated repository of 3D shapes covering a wide range of object categories; we use the official split with 12,137 shapes for training and 2,874 for explaining. Explanations are drawn from the shuffled test list because the raw test files are category ordered; shuffling avoids category-biased prefixes when only a bounded subset is explained. \emph{Optdigits} (64 features) and \emph{Splice} (60 features) are standard tabular benchmarks used with a logistic-regression backbone.

\textbf{Black-Box Model.} For 3D point cloud experiments, we use two target classifiers $f$: PointNet and PointMLP. Both are trained for 200 epochs on ModelNet40 and ShapeNet. PointNet is trained with Adam using batch size 32 and learning rate $10^{-3}$. PointMLP is trained with SGD using batch size 32, initial learning rate $0.1$, momentum $0.9$, weight decay $2 \times 10^{-4}$, and cosine decay to a minimum learning rate of $5 \times 10^{-3}$. We also report results across these backbone choices in Table~\ref{tab:exp_pc}. These experiments evaluate the black-box sufficiency of the returned superpoint coalitions, but they do not establish alignment with architecture-specific receptive fields in PointNet++-style FPS-neighborhood backbones. For tabular experiments, we use a logistic-regression classifier.

\textbf{MaxSAT Solver.} We implement our MiSS search algorithm using PySAT~\cite{DBLP:conf/sat/IgnatievMM18}. Each search iteration rebuilds a weighted CNF and solves it with the RC2 MaxSAT optimizer~\cite{DBLP:journals/jsat/IgnatievMM19}, using MapleSat~\cite{DBLP:conf/sat/LiangGPC16} via PySAT as the underlying SAT backend.

\textbf{Superpoint Construction.} The current point-cloud implementation constructs superpoints with curvature-augmented K-Means rather than FPS sampling or ball-query neighborhoods. For each point, we first estimate a local curvature feature from its Euclidean $k$-nearest-neighbor neighborhood using PCA. We then cluster the augmented features $[x_i,y_i,z_i,\kappa_i]$ with K-Means and convert the resulting hard labels into superpoints $S_j=\{p_i:a_i=j\}$. Because each point receives exactly one K-Means label, the resulting superpoints are disjoint and their union covers the input point cloud; no additional de-duplication or overlap resolution is required. The full construction is given in Algorithm~\ref{alg:superpoint_segmentation}.

\textbf{Perturbation Operator.} In the point-cloud experiments, the perturbation operator $\mathcal{M}$ from Section~\ref{sec:formal_setting} is instantiated as an in-distribution patch-bank replacement. We build the bank from $M_{\mathrm{bank}}=32$ training point clouds by applying the same superpoint construction and storing centered, scale-normalized training patches. During oracle evaluation, superpoints in $C$ are preserved, while each inactive superpoint $S_i \notin C$ is replaced by a sampled bank patch with the same number of points, rescaled to the target region, translated to its centroid, and perturbed by a small random rotation and Gaussian translation. We use this replacement rather than zero padding or point deletion because the black-box classifiers expect fixed-size, plausible point clouds: deletion changes the input size, while zero padding creates artificial origin points and geometric holes. The perturbation therefore tests semantic absence through plausible local replacements rather than out-of-distribution artifacts.

\textbf{Hyperparameters.} We set the dataset-specific point-cloud parameters according to the validation sweep: ModelNet40 uses $K=16$, $L=6$, and perturbation strength $0.4$, while ShapeNet uses $K=40$, $L=5$, and perturbation strength $1.0$. Shared settings follow the current MiSS implementation: RC2 with MapleSat, patch-bank perturbation, the sequential sufficiency oracle, and surrogate-guided proposal. Search uses a cardinality window updated from oracle feedback and invokes the certified exact-size fallback when the heuristic window is exhausted. After each oracle query, the evaluated coalition is added to the history, the blocking clauses are updated, and the surrogate acquisition heuristic is re-fitted before rebuilding the next WCNF. Complete implementation hyperparameters are reported in Appendix~\ref{sec:appendix_hyperparams}.

\textbf{Baselines.} We evaluate our approach against four quantitative baselines: Anchors~\cite{anchor}, LORE~\cite{guidotti2018local}, SkopeRules~\cite{skoperules}, and CORELS~\cite{angelino2017learning}. GradCAM~\cite{gradcam} is included only for visual comparison. Each baseline is implemented following the specifications in its original publication to ensure a fair comparison. When a method is not natively defined for point-cloud inputs, we make the minimal adaptations required to run it on superpoint-based point-cloud representations.

\textbf{Artifact Availability.} The downloadable artifact package for this work is available at \url{https://owncloud.univ-artois.fr/index.php/s/1KfJ1EdzUerUVlS}. The access password is \texttt{lpar\_26\_miss}. We plan to release the source code as an open-source repository in future work.

\subsection{Evaluation Metrics}
We employ four complementary metrics that capture the faithfulness, empirical applicability, parsimony, and stability of the explanation method.

\textbf{Precision.}
Following the sufficiency criterion in Section~\ref{sec:method}, we define $\mathrm{Prec}(C) = \Pr[f(\tilde{P}_{C}) = y]$, the empirical estimate of the probability that the model's prediction is preserved under the perturbation distribution when only coalition $C$ is retained. A valid MiSS must satisfy $\mathrm{Prec}(C) \geq \tau$.
\begin{equation}
	\mathrm{Prec}(C) = \Pr\bigl[f(\tilde{P}_{C}) = y\bigr],
\end{equation}

\textbf{Coverage.}
Coverage measures the empirical applicability of a candidate explanation. Let $\{X_j\}_{j=1}^{N}$ denote evaluation coalitions sampled from the perturbation distribution, represented as subsets of the superpoint universe $\mathcal{S}$. For a candidate MiSS $C$, coverage is the fraction of evaluation coalitions that contain $C$:
\begin{equation}
    \label{eq:coverage}
    \mathrm{Cov}(C)
    =
    \frac{1}{N}\sum_{j=1}^{N}
    \mathbf{1}\!\left[C \subseteq X_j\right],
\end{equation}

\subsection{Evaluation on Point Cloud Benchmarks}

\textbf{Visual Comparison.} As shown in Figure~\ref{fig:exp1}, all explanations are computed at concept-level and rendered back to the signal level for visualization. For all methods except GradCAM, the viridis colormap represents point height along the $z$ axis, while red circles indicate the regions selected as explanations. GradCAM follows its standard signed attribution convention, with red denoting positive contribution and blue denoting negative contribution. For MiSS, the highlighted regions correspond to the superpoints belonging to the returned coalition $C^\star$.

\begin{table*}
\centering
\small
\setlength{\tabcolsep}{4pt}
\begin{tabular}{@{}llcccccc@{}}
\toprule
Dataset & Explainer & \multicolumn{3}{c}{PointMLP} & \multicolumn{3}{c}{PointNet} \\
 & & Prec.(\%)$\uparrow$ & Cov.(\%)$\uparrow$ & Runtime$\downarrow$ & Prec.(\%)$\uparrow$ & Cov.(\%)$\uparrow$ & Runtime$\downarrow$ \\
\midrule
ModelNet40 & Anchors & $58.62$ & $\mathbf{35.16}$ & $52.88$ & $83.32$ & $\mathbf{36.84}$ & $5.88$ \\
 & MiSS & $\mathbf{87.01}$ & $14.88$ & $39.00$ & $90.18$ & $32.12$ & $20.01$ \\
 & MiSS* & $86.93$ & $14.60$ & $22.74$ & $\mathbf{91.43}$ & $31.77$ & $10.14$ \\
 & LORE & $62.41$ & $19.01$ & $26.27$ & $83.08$ & $28.55$ & $3.08$ \\
 & SkopeRules & $65.82$ & $14.80$ & $30.41$ & $84.59$ & $25.01$ & $3.36$ \\
 & CORELS & $64.03$ & $14.99$ & $\mathbf{10.29}$ & $84.76$ & $21.53$ & $\mathbf{1.21}$ \\
\midrule
ShapeNet & Anchors & $93.82$ & $\mathbf{50.11}$ & $20.16$ & $95.83$ & $\mathbf{48.90}$ & $8.34$ \\
 & MiSS & $94.98$ & $45.15$ & $36.09$ & $97.19$ & $46.29$ & $20.90$ \\
 & MiSS* & $\mathbf{95.25}$ & $45.08$ & $\mathbf{4.08}$ & $\mathbf{97.38}$ & $46.62$ & $3.62$ \\
 & LORE & $94.78$ & $43.40$ & $38.11$ & $95.71$ & $44.51$ & $8.30$ \\
 & SkopeRules & $93.47$ & $42.39$ & $39.84$ & $95.98$ & $44.67$ & $8.79$ \\
 & CORELS & $94.17$ & $38.22$ & $11.80$ & $96.07$ & $40.95$ & $\mathbf{2.59}$ \\
\bottomrule
\end{tabular}
\caption{Results on the point cloud benchmarks. We report precision, coverage, and runtime on ModelNet40 and ShapeNet for the PointMLP and PointNet backbones. MiSS* denotes the first sufficient coalition.}
\label{tab:exp_pc}
\end{table*}

\textbf{Quantitative Comparison.} On ModelNet40, MiSS and MiSS* achieve substantially higher precision than all rule-based baselines on both backbones. Their coverage is lower than Anchors but remains competitive with the other baselines. MiSS* approximately halves the runtime of full MiSS while retaining similar precision and coverage.

\textbf{Results.} On ShapeNet, MiSS* achieves the highest precision on both PointMLP and PointNet. Anchors attains the highest coverage, while MiSS and MiSS* rank next and outperform the other rule-based baselines. The similarity between MiSS* and full MiSS indicates that the first sufficient coalition already provides strong explanation quality with substantially lower runtime.

\subsection{Evaluation on Tabular Benchmarks~\label{sec:tabular_benchmarks}}

To test MiSS outside geometric domains, we evaluate it on ten tabular benchmarks using logistic regression as the black-box classifier. We compare against four rule-based or anchor-style explainers: Anchors, LORE, SkopeRules, and CORELS. Table~\ref{tab:exp_tabular} highlights Bank and Breast, while the complete results are reported in Appendix~\ref{sec:appendix_tabular_full}.

\begin{table*}
	\centering
	\small
	\setlength{\tabcolsep}{4pt}
	\begin{tabular}{@{}llcccc@{}}
		\toprule
		Dataset & Explainer  & Dim. & Prec.(\%)$\uparrow$ & Cov.(\%)$\uparrow$ & Runtime (s)$\downarrow$ \\
		\midrule
		Bank    & Anchors    & 16   & $98.13$             & $50.53$            & $0.04$                  \\
		        & MiSS       &      & $\mathbf{99.83}$    & $\mathbf{50.97}$   & $0.62$                  \\
		        & MiSS*      &      & $99.59$             & $49.78$            & $0.61$                  \\
		        & LORE       &      & $98.44$             & $50.22$            & $0.05$                  \\
		        & SkopeRules &      & $98.44$             & $50.22$            & $0.05$                  \\
		        & CORELS     &      & $98.58$             & $49.97$            & $\mathbf{0.02}$         \\
		\midrule
		Breast  & Anchors    & 30   & $85.85$             & $51.06$            & $0.03$                  \\
		        & MiSS       &      & $79.78$             & $48.81$            & $0.03$                  \\
		        & MiSS*      &      & $79.22$             & $50.56$            & $\mathbf{0.01}$         \\
		        & LORE       &      & $80.55$             & $\mathbf{51.28}$   & $0.06$                  \\
		        & SkopeRules &      & $80.55$             & $\mathbf{51.28}$   & $0.06$                  \\
		        & CORELS     &      & $\mathbf{91.42}$    & $39.03$            & $0.02$                  \\
		\bottomrule
	\end{tabular}
	\caption{Main tabular experiment summary on Bank and Breast for the logistic-regression backbone. Dim. reports the number of features.}
	\label{tab:exp_tabular}
\end{table*}

\textbf{Results.} On Bank, MiSS achieves the highest precision (99.83\%) and coverage (50.97\%), but is slower than the rule-based baselines. On Breast, MiSS* is the fastest method and retains competitive coverage, although its precision trails CORELS. Across all ten datasets in Table~\ref{tab:exp_tabular_full}, MiSS is competitive on precision and coverage, while MiSS* is the fastest or tied-fastest method on seven datasets. Overall, the results show a dataset-dependent trade-off between explanation quality and runtime.

\subsection{Ablation Study}
\begin{figure*}[t]
	\centering
	\includegraphics[width=1\linewidth]{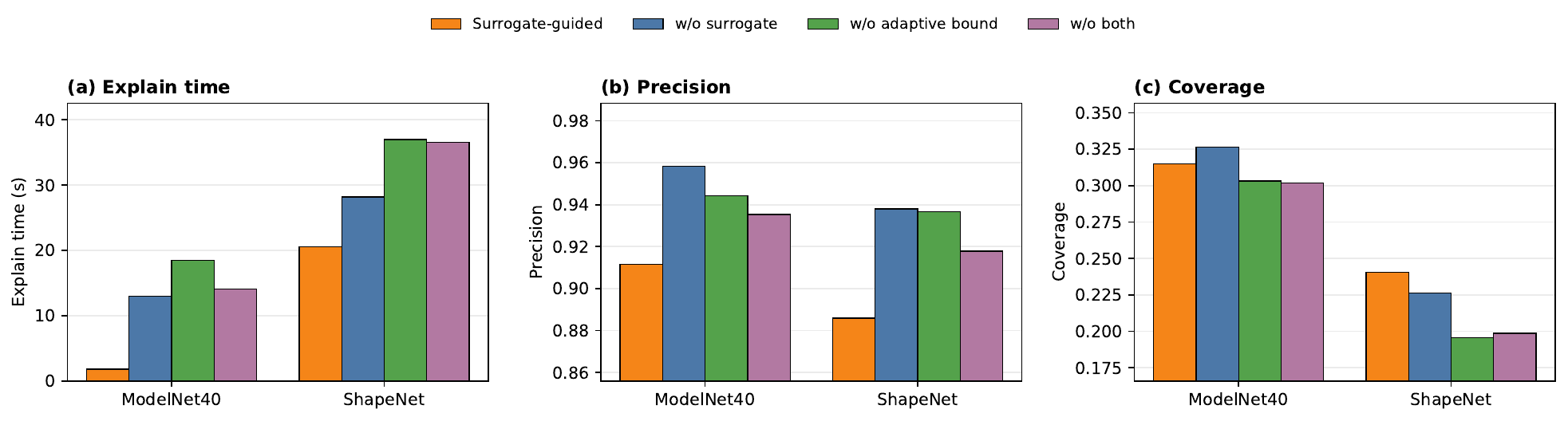}
	\caption{Comparison of the runtime, precision, and coverage of the first verified explanation returned by full MiSS and variants that remove surrogate guidance, dynamic cardinality bounds, or both.}
	\label{surrogate_ablation_summary}
\end{figure*}

\paragraph{Effect of Surrogate-Guided Proposal}
Figure~\ref{surrogate_ablation_summary} compares the default method with variants that remove surrogate guidance, dynamic cardinality bounds, or both. The default method has the lowest explanation runtime on both datasets while retaining comparable precision and coverage.

\paragraph{Search-Budget Sensitivity}
\begin{figure*}[tb]
	\centering
	\includegraphics[width=\linewidth]{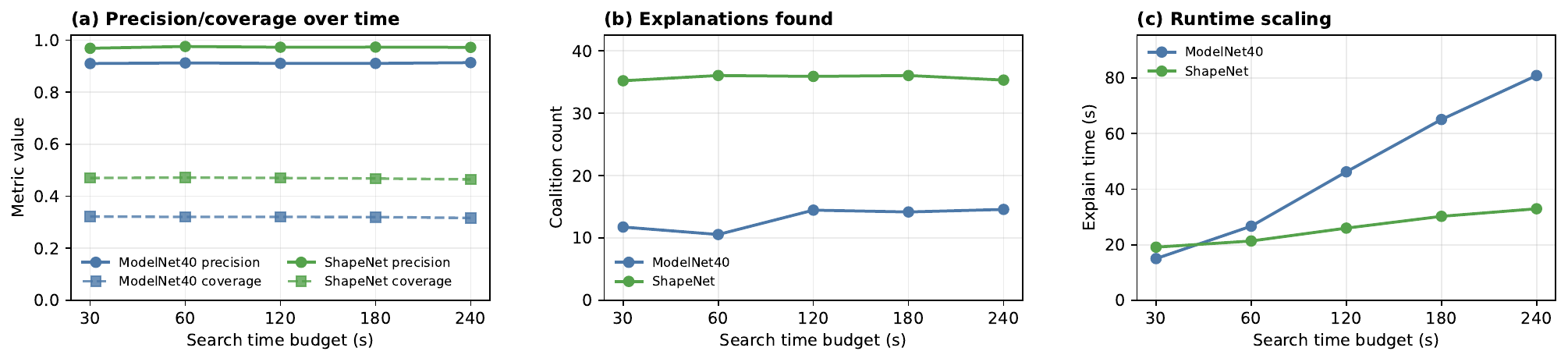}
	\caption{Search-budget sensitivity on ModelNet40 and ShapeNet. Panels (b)--(d) report precision and coverage, the number of coalitions found, and explanation runtime as the search-time budget increases.}
	\label{fig:search_budget_sensitivity}
\end{figure*}

Figure~\ref{fig:search_budget_sensitivity} reports results from sweeping $T_{\max}$ from $30$ to $240$ seconds. Precision, coverage, and coalition count remain nearly flat, whereas explanation runtime increases with the budget; longer search therefore provides little improvement in explanation quality or yield.

\paragraph{Solver Sensitivity}
\begin{table*}
	\centering
	\small
	\setlength{\tabcolsep}{4pt}
	\begin{tabular}{@{}lcccccccc@{}}
		\toprule
		Metric               & Glucose3 & Glucose4 & MapleSat         & MapleChrono & MapleCM & Minisat22 & MinisatGH \\
		\midrule
		Time (s) $\downarrow$ & $32.57$  & $34.36$  & $\mathbf{31.69}$ & $33.06$     & $34.25$ & $34.25$   & $35.88$   \\
		\bottomrule
	\end{tabular}
	\caption{Average MiSS explanation runtime under different PySAT solvers, aggregated over all backbone/dataset settings. Lower is better.}
	\label{tab:solver_ablation_runtime_compact}
\end{table*}

Table~\ref{tab:solver_ablation_runtime_compact} compares the average explanation runtime under different PySAT solvers.
MapleSat gives the lowest average runtime, followed by Glucose3 and MapleChrono, so we use MapleSat as the default solver.
The complete per-backbone and per-dataset solver results are reported in Appendix~\ref{sec:appendix_ablation}.

\paragraph{Hyperparameter Sensitivity}
\begin{figure*}[tb]
	\centering
	\includegraphics[width=1\linewidth]{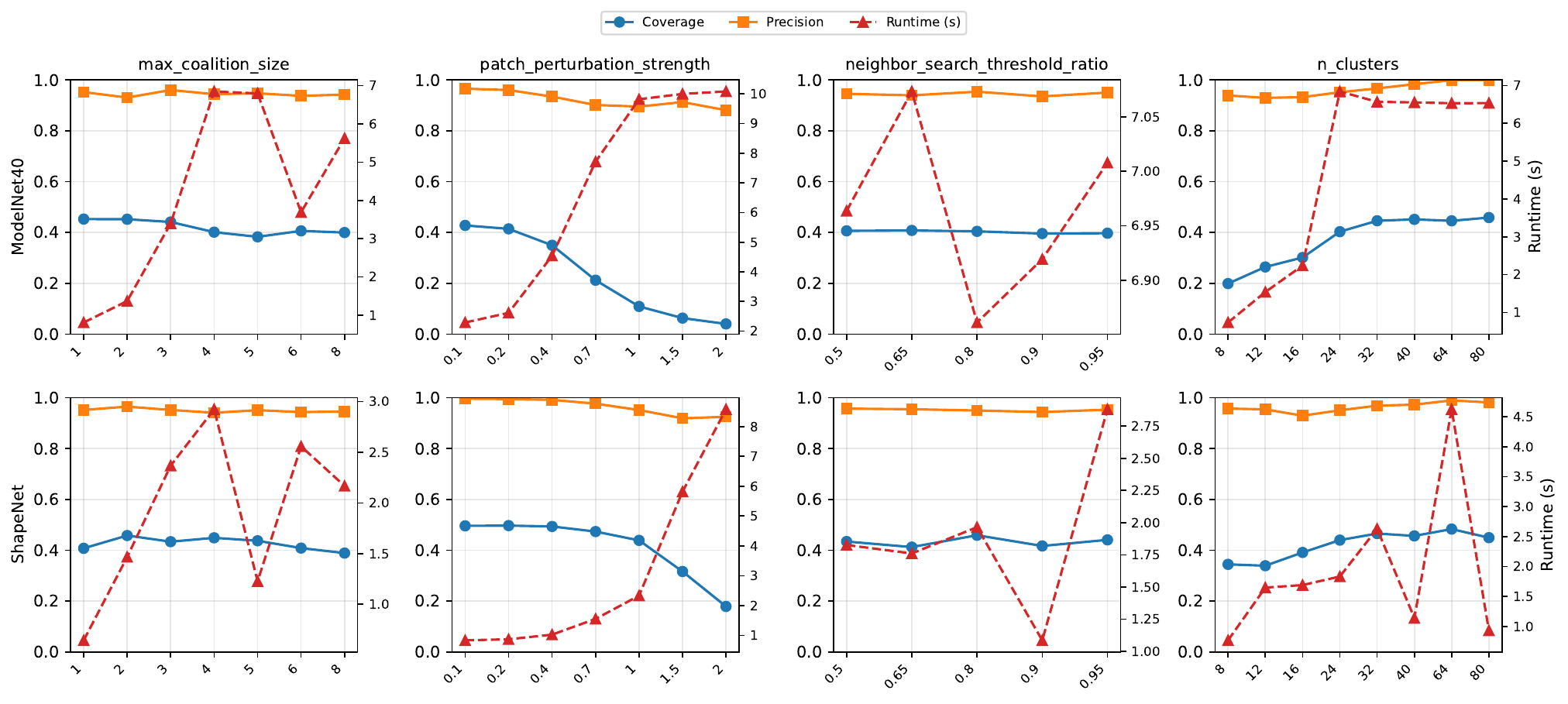}
	\caption{Hyperparameter sensitivity on ModelNet40 (top) and ShapeNet (bottom). Each column varies one hyperparameter while keeping the others fixed. Precision and coverage use the left $y$-axis, while runtime uses the right $y$-axis.}
	\label{fig:pc_ablation}
\end{figure*}

Figure~\ref{fig:pc_ablation} shows the point-cloud hyperparameter sweep over maximum coalition size, perturbation strength, neighbor-search threshold ratio, and superpoint cluster number. Accordingly, we set $L=6$, perturbation strength $0.4$, threshold ratio $0.8$, and $K=16$ for ModelNet40, and $L=5$, perturbation strength $1.0$, threshold ratio $0.9$, and $K=40$ for ShapeNet.

\section{Conclusion} \label{sec:conclusion}

We presented MiSS, a black-box framework that searches for perturbation-relative sufficient coalitions over a chosen superpoint partition. MiSS combines surrogate-guided weighted MaxSAT proposal with oracle verification, blocking clauses, a heuristic adaptive cardinality floor, certified exact-size fallback, and a safely tightened upper bound. Experiments on point-cloud and tabular benchmarks demonstrate its ability to produce compact, high-precision explanations without requiring a logical encoding of the classifier.

\textbf{Scope and limitations.}
The guarantees are relative to the partition, perturbation operator, threshold $\tau$, and certified-termination condition in Theorem~\ref{thm:minimality}. Although the heuristic floor may temporarily skip shorter coalitions, the exact-size fallback revisits skipped levels before certified termination. Minimum cardinality is therefore guaranteed when certification completes, but not when the time or query budget interrupts it. The explanations also depend on the semantic quality of the partition and perturbations, and should be interpreted as partition-level sufficiency evidence rather than causal evidence, internal model units, or a safety certificate.

Future work will study architecture-aware regions and the use of MiSS explanations for retraining, data curation, and downstream decision support.

\clearpage
\bibliographystyle{plain}
\bibliography{ref}

\clearpage
\appendix
\section{Full Tabular Benchmark Results}
\label{sec:appendix_tabular_full}

Table~\ref{tab:exp_tabular_full} reports the complete tabular benchmark results over all evaluated datasets, complementing the summarized results in the main text.

\begin{table}[H]
	\centering
	\small
	\setlength{\tabcolsep}{4pt}
	\begin{tabular}{@{}llcccc@{}}
		\toprule
		Dataset   & Explainer  & Dim. & Prec.(\%)$\uparrow$ & Cov.(\%)$\uparrow$ & Runtime (s)$\downarrow$ \\
		\midrule
		Bank      & Anchors    & 16   & $98.13$             & $50.53$            & $0.04$              \\
		          & MiSS       &      & $\mathbf{99.83}$    & $\mathbf{50.97}$   & $0.62$              \\
		          & MiSS*      &      & $99.59$             & $49.78$            & $0.61$              \\
		          & LORE       &      & $98.44$             & $50.22$            & $0.05$              \\
		          & SkopeRules &      & $98.44$             & $50.22$            & $0.05$              \\
		          & CORELS     &      & $98.58$             & $49.97$            & $\mathbf{0.02}$     \\
		\midrule
		Breast    & Anchors    & 30   & $85.85$             & $51.06$            & $0.03$              \\
		          & MiSS       &      & $79.78$             & $48.81$            & $0.03$              \\
		          & MiSS*      &      & $79.22$             & $50.56$            & $\mathbf{0.01}$     \\
		          & LORE       &      & $80.55$             & $\mathbf{51.28}$   & $0.06$              \\
		          & SkopeRules &      & $80.55$             & $\mathbf{51.28}$   & $0.06$              \\
		          & CORELS     &      & $\mathbf{91.42}$    & $39.03$            & $0.02$              \\
		\midrule
		Credit    & Anchors    & 23   & $99.66$             & $48.53$            & $0.03$              \\
		          & MiSS       &      & $99.49$             & $\mathbf{50.22}$   & $0.03$              \\
		          & MiSS*      &      & $99.65$             & $49.00$            & $\mathbf{0.00}$     \\
		          & LORE       &      & $99.43$             & $48.16$            & $0.05$              \\
		          & SkopeRules &      & $99.43$             & $48.16$            & $0.06$              \\
		          & CORELS     &      & $\mathbf{99.94}$    & $48.28$            & $0.02$              \\
		\midrule
		Heart     & Anchors    & 13   & $86.11$             & $49.49$            & $0.03$              \\
		          & MiSS       &      & $88.96$             & $49.95$            & $0.02$              \\
		          & MiSS*      &      & $87.26$             & $49.80$            & $\mathbf{0.01}$     \\
		          & LORE       &      & $86.75$             & $\mathbf{50.77}$   & $0.05$              \\
		          & SkopeRules &      & $86.75$             & $\mathbf{50.77}$   & $0.06$              \\
		          & CORELS     &      & $\mathbf{93.66}$    & $43.49$            & $0.03$              \\
		\midrule
		Mushroom  & Anchors    & 22   & $84.34$             & $\mathbf{50.88}$   & $0.03$              \\
		          & MiSS       &      & $81.24$             & $49.12$            & $0.03$              \\
		          & MiSS*      &      & $75.26$             & $\mathbf{50.88}$   & $\mathbf{0.01}$     \\
		          & LORE       &      & $80.48$             & $50.00$            & $0.05$              \\
		          & SkopeRules &      & $80.48$             & $50.00$            & $0.06$              \\
		          & CORELS     &      & $\mathbf{92.73}$    & $39.44$            & $0.02$              \\
		\bottomrule
	\end{tabular}
	\caption{Full tabular experiment results for the logistic-regression backbone. Dim. reports the number of features. MiSS* denotes the first sufficient coalition.}
	\label{tab:exp_tabular_full}
\end{table}

\begin{table}[H]
	\ContinuedFloat
	\centering
	\small
	\setlength{\tabcolsep}{4pt}
	\begin{tabular}{@{}llcccc@{}}
		\toprule
		Dataset   & Explainer  & Dim. & Prec.(\%)$\uparrow$ & Cov.(\%)$\uparrow$ & Runtime (s)$\downarrow$ \\
		\midrule
		Optdigits & Anchors    & 64   & $74.62$             & $\mathbf{46.47}$   & $2.23$              \\
		          & MiSS       &      & $81.53$             & $31.00$            & $34.92$             \\
		          & MiSS*      &      & $\mathbf{81.94}$    & $30.59$            & $11.99$             \\
		          & LORE       &      & $71.14$             & $\mathbf{46.47}$   & $0.08$              \\
		          & SkopeRules &      & $73.16$             & $44.47$            & $0.09$              \\
		          & CORELS     &      & $79.31$             & $27.56$            & $\mathbf{0.03}$     \\
		\midrule
		Splice    & Anchors    & 60   & $85.07$             & $49.44$            & $0.04$              \\
		          & MiSS       &      & $86.63$             & $48.62$            & $0.11$              \\
		          & MiSS*      &      & $87.95$             & $50.88$            & $\mathbf{0.02}$     \\
		          & LORE       &      & $87.27$             & $\mathbf{51.53}$   & $0.08$              \\
		          & SkopeRules &      & $87.27$             & $\mathbf{51.53}$   & $0.09$              \\
		          & CORELS     &      & $\mathbf{89.79}$    & $39.69$            & $0.02$              \\
		\midrule
		Vowel     & Anchors    & 12   & $68.43$             & $\mathbf{42.91}$   & $0.16$              \\
		          & MiSS       &      & $\mathbf{77.25}$    & $36.47$            & $0.68$              \\
		          & MiSS*      &      & $73.72$             & $35.59$            & $0.18$              \\
		          & LORE       &      & $65.88$             & $41.72$            & $0.05$              \\
		          & SkopeRules &      & $68.95$             & $39.72$            & $0.05$              \\
		          & CORELS     &      & $75.58$             & $31.06$            & $\mathbf{0.02}$     \\
		\midrule
		Wine      & Anchors    & 13   & $90.05$             & $47.14$            & $0.04$              \\
		          & MiSS       &      & $91.94$             & $50.43$            & $0.02$              \\
		          & MiSS*      &      & $88.24$             & $\mathbf{50.87}$   & $\mathbf{0.01}$     \\
		          & LORE       &      & $90.45$             & $48.70$            & $0.06$              \\
		          & SkopeRules &      & $90.45$             & $48.70$            & $0.06$              \\
		          & CORELS     &      & $\mathbf{98.77}$    & $48.61$            & $0.03$              \\
		\midrule
		Quality   & Anchors    & 11   & $84.52$             & $\mathbf{50.44}$   & $0.04$              \\
		          & MiSS       &      & $86.48$             & $47.47$            & $0.05$              \\
		          & MiSS*      &      & $83.97$             & $48.50$            & $\mathbf{0.02}$     \\
		          & LORE       &      & $85.02$             & $50.34$            & $0.04$              \\
		          & SkopeRules &      & $85.02$             & $50.34$            & $0.05$              \\
		          & CORELS     &      & $\mathbf{94.62}$    & $44.97$            & $0.02$              \\
		\bottomrule
	\end{tabular}
	\caption{Full tabular experiment results (continued).}
\end{table}

\section{Point-Cloud Hyperparameters}
\label{sec:appendix_hyperparams}

Table~\ref{tab:hyperparams} lists the point-cloud implementation hyperparameters used for ModelNet40 and ShapeNet.

\begin{table*}[t]
	\centering
	\small
	\setlength{\tabcolsep}{6pt}
	\begin{tabular}{@{}llcc@{}}
		\toprule
		Component   & Hyperparameter                         & ModelNet40         & ShapeNet          \\
		\midrule
		\multirow{6}{*}{Superpoints} & Number of segments $K$                    & $16$              & $40$              \\
		                             & Curvature neighbors                       & \multicolumn{2}{c}{$20$}              \\
		                             & Clustering method                         & \multicolumn{2}{c}{K-Means on $[x,y,z,\kappa]$} \\
		                             & K-Means restarts                          & \multicolumn{2}{c}{$10$}              \\
		                             & K-Means max iterations                    & \multicolumn{2}{c}{$300$}             \\
		                             & K-Means tolerance                         & \multicolumn{2}{c}{$10^{-4}$}         \\
		\multirow{4}{*}{Masker}      & Inactive replacement mode                 & \multicolumn{2}{c}{\texttt{patch\_bank}} \\
		                             & Mask activation probability               & \multicolumn{2}{c}{$0.5$}             \\
		                             & Patch-bank train samples                  & \multicolumn{2}{c}{$32$}              \\
		                             & Patch perturbation strength               & $0.4$             & $1.0$             \\
		\multirow{7}{*}{Search}      & Max coalition size $L$                    & $6$               & $5$               \\
		                             & Successful-coalition cap                  & \multicolumn{2}{c}{$500$}             \\
		                             & Time limit $T_{\max}$                     & \multicolumn{2}{c}{$60$s}             \\
		                             & MaxSAT optimizer                          & \multicolumn{2}{c}{RC2}               \\
		                             & SAT backend                               & \multicolumn{2}{c}{MapleSat}          \\
		                             & Heuristic adaptive cardinality floor      & \multicolumn{2}{c}{enabled}           \\
		                             & Certified exact-size fallback             & \multicolumn{2}{c}{enabled}           \\
		                             & Floor patience $\eta$                     & \multicolumn{2}{c}{$8$}               \\
		\multirow{5}{*}{Oracle}      & Precision threshold $\tau$                & \multicolumn{2}{c}{$0.85$}            \\
		                             & Confidence level $\delta_{\mathrm{test}}$ & \multicolumn{2}{c}{$0.05$}            \\
		                             & Oracle batch size $B$                     & \multicolumn{2}{c}{$100$}             \\
		                             & Max oracle batches                        & \multicolumn{2}{c}{$5$}               \\
		                             & Max oracle samples $N_{\max}$             & \multicolumn{2}{c}{$500$}             \\
		\multirow{3}{*}{Surrogate}   & Surrogate-guided proposal                 & \multicolumn{2}{c}{enabled}           \\
		                             & Ridge penalty                             & \multicolumn{2}{c}{$1.0$}             \\
		                             & Soft-weight scale                         & \multicolumn{2}{c}{$100$}             \\
		\multirow{2}{*}{Selection}   & Primary coverage tolerance                & \multicolumn{2}{c}{$0.02$}            \\
		                             & Primary precision margin                  & \multicolumn{2}{c}{$0.05$}            \\
		\multirow{1}{*}{General}     & Random seed                               & \multicolumn{2}{c}{$43$}              \\
		\bottomrule
	\end{tabular}
	\caption{Main hyperparameters used by the point-cloud MiSS explainer. Dataset-specific settings match the current code path; shared settings are identical for ModelNet40 and ShapeNet. The oracle sample cap satisfies $N_{\max}=B$ times the maximum number of oracle batches. Training details of the black-box models are reported separately in the text.}
	\label{tab:hyperparams}
\end{table*}

\section{Algorithm Listings}
\label{sec:appendix_algorithms}

Algorithm~\ref{alg:superpoint_segmentation} gives the curvature-augmented K-Means procedure used to construct the superpoint partition in the point-cloud experiments.

\begin{algorithm}[!t]
	\caption{Curvature-Augmented Superpoint Segmentation}
	\label{alg:superpoint_segmentation}
	\begin{algorithmic}[1]
		\Require Point cloud $P=\{p_i\}_{i=1}^{N}$, number of superpoints $K$, curvature neighborhood size $k$, random seed $s$
		\Ensure Superpoint partition $\mathcal{S}=\{S_j\}_{j=1}^{K}$

		\State Fit a Euclidean $k$-nearest-neighbor index on $P$
		\For{$i=1$ \To $N$}
			\State Let $\mathcal{N}_k(i)$ be the $k$ nearest neighbors of $p_i$
			\State Compute the covariance matrix of the points in $\mathcal{N}_k(i)$
			\State Let $\lambda_0 \leq \lambda_1 \leq \lambda_2$ be the covariance eigenvalues
			\State Set $\kappa_i \leftarrow \lambda_0 / (\lambda_0+\lambda_1+\lambda_2+\epsilon)$
			\State Form the augmented feature $u_i \leftarrow [x_i,y_i,z_i,\kappa_i]$
		\EndFor
		\State Run K-Means on $\{u_i\}_{i=1}^{N}$ with $K$ clusters, seed $s$, $10$ restarts, $300$ maximum iterations, and tolerance $10^{-4}$
		\State Let $a_i \in \{1,\ldots,K\}$ be the hard K-Means label of point $p_i$
		\For{$j=1$ \To $K$}
			\State Set $S_j \leftarrow \{p_i \in P: a_i=j\}$
		\EndFor
		\State \Return $\mathcal{S}=\{S_1,\ldots,S_K\}$
	\end{algorithmic}
\end{algorithm}

\section{Ablation Study}
\label{sec:appendix_ablation}

\textbf{Solver Runtime.} Table~\ref{tab:solver_ablation_runtime} reports the complete solver-runtime ablation across all backbone and dataset settings, complementing the average-runtime summary in Table~\ref{tab:solver_ablation_runtime_compact}.

\begin{table*}
	\centering
	\small
	\setlength{\tabcolsep}{4pt}
	\begin{tabular}{@{}lccccc@{}}
		\toprule
		Solver      & PointMLP@MN40    & PointMLP@SN      & PointNet@MN40    & PointNet@SN      & Avg.             \\
		\midrule
		Glucose3    & $54.10$          & $34.36$          & $23.38$          & $18.43$          & $32.57$          \\
		Glucose4    & $55.43$          & $34.47$          & $29.08$          & $18.45$          & $34.36$          \\
		Glucose42   & $55.24$          & $\mathbf{34.07}$ & $29.54$          & $18.27$          & $34.28$          \\
		MapleSat    & $\mathbf{50.88}$ & $34.72$          & $\mathbf{22.91}$ & $18.24$          & $\mathbf{31.69}$ \\
		MapleChrono & $50.89$          & $34.10$          & $28.90$          & $18.35$          & $33.06$          \\
		MapleCM     & $54.76$          & $35.13$          & $28.86$          & $\mathbf{18.23}$ & $34.25$          \\
		Minisat22   & $54.59$          & $34.68$          & $29.48$          & $18.26$          & $34.25$          \\
		MinisatGH   & $54.79$          & $34.09$          & $36.35$          & $18.31$          & $35.88$          \\
		\bottomrule
	\end{tabular}
	\caption{MiSS explanation runtime under different PySAT solvers. Values are mean explanation time in seconds; lower is better. MN40 denotes ModelNet40, SN denotes ShapeNet.}
	\label{tab:solver_ablation_runtime}
\end{table*}

\section{Manifold-Preserving Perturbation}\label{sec:manifold_perturbation}
To verify the sufficiency condition defined above, we must simulate the removal of features. Zero-padding can create ``holes'' in the 3D structure, pushing the data out-of-distribution.
To ensure the perturbed point cloud $\tilde{P}_{C}$ remains on the data manifold, we utilize a \textit{neighborhood-based sampling strategy}.
For each masked superpoint $S_i \notin C$, we replace its points with samples drawn from the local geometric context:
\begin{equation}
	\tilde{S}_i = \{ p \mid p \sim \mathcal{U}(\bigcup_{j \in \mathcal{N}(i)} S_j) \}
\end{equation}
where $\mathcal{N}(i)$ denotes the set of $k$-nearest unmasked superpoints to the centroid of $S_i$. This preserves the local geometric density, ensuring that the classifier evaluates the \textit{semantic} absence of the feature rather than an artifact of the deletion process.

\end{document}